\documentclass[journal,twoside,web]{ieeecolor}
\usepackage{generic}
\usepackage{cite}
\usepackage{amsmath,amssymb,amsfonts}
\usepackage{algorithmic}
\usepackage{graphicx}
\usepackage{algorithm,algorithmic}
\usepackage{hyperref}
\hypersetup{hidelinks=true}
\usepackage{textcomp}
\usepackage{multirow}
\usepackage{bbding}
\newtheorem{theorem}{Theorem}[section]
\newtheorem{lemma}[theorem]{Lemma}
\newtheorem{corollary}[theorem]{Corollary}
\def\BibTeX{{\rm B\kern-.05em{\sc i\kern-.025em b}\kern-.08em
    T\kern-.1667em\lower.7ex\hbox{E}\kern-.125emX}}
\begin{document}
\title{Topological Conditioning for Mammography Models via a Stable Wavelet-Persistence Vectorization}
\author{Charles Fanning and Mehmet Emin Aktas
\thanks{ }
\thanks{ }
\thanks{ }
}

\maketitle

\begin{abstract}
Breast cancer is the most commonly diagnosed cancer in women and a leading cause of cancer death worldwide. Screening mammography reduces mortality, yet interpretation still suffers from substantial false negatives and false positives, and model accuracy often degrades when deployed across scanners, modalities, and patient populations. We propose a simple conditioning signal aimed at improving external performance based on a wavelet based vectorization of persistent homology. Using topological data analysis, we summarize image structure that persists across intensity thresholds and convert this information into spatial, multi scale maps that are provably stable to small intensity perturbations. These maps are integrated into a two stage detection pipeline through input level channel concatenation. The model is trained and validated on the CBIS DDSM digitized film mammography cohort from the United States and evaluated on two independent full field digital mammography cohorts from Portugal (INbreast) and China (CMMD), with performance reported at the patient level. On INbreast, augmenting ConvNeXt Tiny with wavelet persistence channels increases patient level AUC from 0.55 to 0.75 under a limited training budget.
\end{abstract}


\section{Introduction}
\label{sec:introduction}

\IEEEPARstart{B}{reast} cancer is the most commonly diagnosed cancer in women and a leading cause of cancer mortality worldwide. Screening mammography reduces deaths, but interpretation remains imperfect: substantial \emph{false negatives} (missed cancers) and \emph{false positives} (unnecessary recalls or biopsies) persist in routine practice, and performance varies across readers, scanners, and sites. At the screening scale, these errors translate into patient harm, increased workload, and rising cost.

Despite many reports of strong \emph{within--dataset} performance, three practical limitations constrain the reliability of deployed systems. First, transportability: models trained on a single cohort often exhibit degraded performance under scanner, modality, and population shift, even at large training scales \cite{Kooi2017MammogramCAD,McKinney2020InternationalAI}. Second, explanation fragility: popular saliency and attention methods frequently fail basic faithfulness tests and can vary substantially under small input perturbations or domain shift \cite{Arun2021SaliencyTrustworthiness,Fuhrman2022ExplainableAIReview,10208407,HASSAN2025109569}. Third, morphology sensitivity: lesion--salient microtexture, spiculation, and architectural distortion are inherently multi--scale and orientation--dependent, yet generic convolutional pipelines do not explicitly encode these structures as first--class signals \cite{Eltoukhy2010WaveletCurvelet,Oyelade2022WaveletCNN,RAGHAVENDRA2016151}.

In this work, we use \emph{topological data analysis} to summarize image structure that \emph{persists} across intensity thresholds and to convert this information into spatial, multi–scale channels defined on the image grid. The construction provides a stable representation of coarse and fine morphological variation and yields deterministic auxiliary features that can be visualized alongside the original mammogram.

We incorporate these wavelet–persistence channels as an additional morphology input within a standard two–stage mammography pipeline based on patch–level multiple–instance learning followed by detection. The channels are concatenated with the grayscale image at the network input, and no changes are made to the backbone architecture, losses, or optimization procedures.

We adopt a fixed–source, fixed–budget evaluation design: all training and validation occur on the CBIS–DDSM cohort, and generalization is assessed on two independent full–field digital mammography cohorts, INbreast and CMMD. INbreast is used to evaluate external discrimination under topology conditioning, while CMMD is used to examine the sensitivity of the wavelet–persistence construction through controlled ablations.

Under this protocol, topological conditioning changes external performance for the single controlled probe architecture. On INbreast, augmenting ConvNeXt–Tiny with wavelet–persistence channels increases patient–level AUC from $0.55 \pm 0.11$ to $0.75 \pm 0.09$ under an identical training budget. CMMD serves only as an ablation domain in this study and is not used for a direct grayscale–versus–topology comparison.

The contributions of this work are as follows:
\begin{enumerate}
  \item We develop a wavelet–persistence representation that maps cubical persistent homology into spatial, multi–scale channels on the image grid and establish a global Lipschitz bound of this mapping with respect to the $1$–Wasserstein distance on persistence diagrams (Section~\ref{sec:methods}).
  \item We incorporate these topological channels into a standard two–stage mammography pipeline through input–level channel concatenation, leaving the backbone architecture, losses, and optimization unchanged (Section~\ref{sec:methods}).
  \item We evaluate this conditioning under a fixed–source, fixed–budget external–validation design: all training occurs on CBIS–DDSM, external discrimination is assessed on INbreast, and CMMD is used for controlled ablations of the wavelet–persistence construction (Section~\ref{sec:experiments-and-results}).
\end{enumerate}

\section{Related Works}\label{sec:related-works}

\subsection{Deep Learning for Mammography}\label{subsec:Deep-Learning-for-Mammography}
Deep learning has become the dominant approach for lesion detection and breast-level cancer prediction in full-field digital mammography (FFDM). Ribli \emph{et al.} proposed a Faster R-CNN system with a VGG-16 backbone that jointly detects and classifies benign versus malignant findings, achieving a breast-level AUC of approximately $0.85$ in the Digital Mammography DREAM Challenge and image-level AUC near $0.95$ on INbreast at high sensitivity with few false positives per image \cite{Ribli2018MammogramDL}. 

In a large head-to-head comparison involving roughly 45{,}000 images, Kooi \emph{et al.} showed that a CNN surpasses a mature, handcrafted-feature CAD system at low-sensitivity operating points and remains competitive at high sensitivity. Adding contextual cues such as lesion location, symmetry, and patient metadata further improved specificity, and a patch-level reader study found no significant difference between the network and certified screening radiologists \cite{Kooi2017MammogramCAD}. 

McKinney \emph{et al.} evaluated an AI system on representative UK and US screening cohorts and reported absolute reductions in false positives of 1.2\% (UK) and 5.7\% (US) and in false negatives of 2.7\% (UK) and 9.4\% (US) relative to routine clinical reads. Their study also demonstrated cross-population generalization (trained on UK, tested on US data) and an 88\% reduction in second-reader workload in a simulated UK double-reading workflow without degrading accuracy \cite{McKinney2020InternationalAI}.

\subsection{Multi-Scale and Morphological Feature Modeling}\label{subsec:Multi-Scale-and-Morphological-Feature-Modeling}
Before modern CNN-based systems, multiresolution descriptors—including wavelets, curvelets, and Gabor filters—were central to mammography CAD because they capture spiculation, architectural distortion, and fine microtexture across multiple scales and orientations. Eltoukhy \emph{et al.} compared discrete wavelet and curvelet representations on digitized mammograms and showed that directional multiscale atoms provide discriminatory power beyond raw intensity features \cite{Eltoukhy2010WaveletCurvelet}. Raghavendra \emph{et al.} combined Gabor wavelets with Locality Sensitive Discriminant Analysis to encode morphological and textural structure for automated identification of breast abnormalities \cite{RAGHAVENDRA2016151}.

More recently, hybrid architectures have incorporated frequency-domain priors directly into deep learning pipelines. Oyelade \emph{et al.} proposed a wavelet–CNN–wavelet framework that combines seam-carving and wavelet preprocessing with learnable wavelet-transform layers and GAN-based augmentation for architectural distortion, reporting improved classification performance on CBIS--DDSM and MIAS \cite{Oyelade2022WaveletCNN}. These methods illustrate continued interest in embedding multiscale and orientation-sensitive structure within mammographic analysis networks.

\subsection{Topological Data Analysis in Medical Imaging}\label{subsec:Topological-Data-Analysis-in-Medical-Imaging}
Topological Data Analysis (TDA) provides shape-based representations that are stable under deformations and complementary to intensity- and texture-based features. In medical image segmentation, Byrne \emph{et al.} introduced a persistent-homology-based topological loss that enforces multi-class anatomical plausibility (connectedness and loops/voids) via cubical complexes, improving topological correctness while preserving overlap-based metrics on 2D short-axis and 3D whole-heart cardiac MR images \cite{Byrne_2023}. 

For medical image classification, Peng \emph{et al.} proposed PHG-Net, in which cubical persistence diagrams are encoded using a PointNet-style module and fused with CNN/Transformer features through gated refinement and a joint topological–vision loss, yielding consistent gains across multiple medical image benchmarks \cite{10484262}. In the context of mammography, Malek \emph{et al.} employed persistent homology both as a noise-aware filter on persistence diagrams and as a vectorization source (e.g., persistent images and persistent entropy) for classifying microcalcification patches from the MIAS and DDSM datasets, demonstrating that persistence filtering and representation choices influence downstream classification performance \cite{Malek2023PersistentHomologyMammo}. Collectively, these studies establish persistent homology as a viable source of auxiliary structure for medical imaging tasks ranging from segmentation to classification.

\begin{figure*}[!t]
\centering
\includegraphics[width=.75\textwidth]{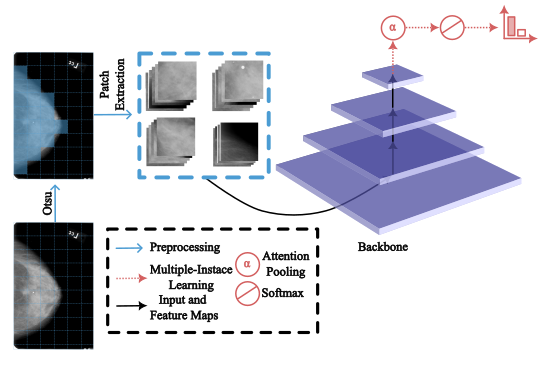}
\caption{During Stage~1, a mammogram is preprocessed using an Otsu-based breast mask and partitioned into fixed-size patches. These patches form the input to a patch-level multiple-instance learning (MIL) backbone that produces coarse image-level predictions. The resulting patch embeddings and breast-level scores serve as inputs to the downstream detection pipeline illustrated in Fig.~\ref{fig:method_framework_stage2}.}
\label{fig:method_framework_stage1}
\end{figure*}

\section{Methods}\label{sec:methods}

Let $I\in[0,1]^{m\times n}$ be a grayscale image on $\Omega=\{1,\ldots,m\}\times\{1,\ldots,n\}$, equipped with its standard cubical structure. For any threshold $\delta\in[0,1]$, the sublevel set
\[
K_\delta = \{(x,y)\in\Omega : I(x,y)\le \delta\}
\]
defines a nested sequence of cubical complexes $K_{\delta_1}\subseteq K_{\delta_2}\subseteq\cdots$. We compute cubical persistent homology over $\mathbb{Z}_2$ in dimensions $0$ and $1$, yielding persistence diagrams $D(I)$ consisting of birth--death pairs $(b,d)$ with $0\le b<d\le 1$. Classical stability results imply that if $\|I-\tilde I\|_\infty\le\varepsilon$, then
$d_B(D(I), D(\tilde I))\le\varepsilon$ and, more generally, $W_p^{(\infty)}(D(I), D(\tilde I))\le\varepsilon$ for all $p\in[1,\infty]$ \cite{CohenSteiner2007StabilityPD}.

In practice, persistence is computed on an aspect-preserving downsampled image $\widehat I$ (maximum side length $96$~pixels) to control memory and runtime. For $H_0$, we retain all finite bars except the dominant one corresponding to the global background component. For $H_1$, we apply a simple persistence-based truncation: letting $D_1(I)$ denote the $H_1$ diagram with lifetimes $\ell(b,d)=d-b$, we sort the bars by $\ell$ and keep the top $\lceil h_{1,\mathrm{pct}}\cdot |D_1(I)|\rceil$ elements for a user-specified retention level $h_{1,\mathrm{pct}}\in(0,1]$ (see ablations in Section~\ref{sec:experiments-and-results}). Because the stability bounds apply to $D(\widehat I)$, and the vectorization introduced below is Lipschitz with respect to diagram perturbations, small changes in the underlying persistence diagrams translate into controlled changes in the resulting spatial maps.

Figure~\ref{fig:method_framework_stage1} illustrates \emph{Stage~I}, a patch-level multiple-instance learning (MIL) model trained only on grayscale images. Each mammogram is min--max normalized to $[0,1]$ and embedded in a fixed grid. We then extract overlapping $224{\times}224$ patches that form a bag associated with a single breast-level label. A convolutional backbone processes each patch independently and the resulting patch scores are aggregated through a shallow feature-pyramid/MIL head to produce a single malignancy score for the image. Stage~I therefore learns a coarse spatial distribution of suspicious regions and a breast-level prior for malignancy, without using any topological information.

Figure~\ref{fig:method_framework_stage2} depicts \emph{Stage~II}, where we construct wavelet--persistence maps and inject them into a two-stage detector. For each mammogram, we compute cubical persistent homology in $H_0$ and $H_1$ on a downsampled version of the image, map the resulting birth--death pairs into a multi-level wavelet pyramid, and obtain spatial, multi-scale wavelet--persistence channels $T(X)$ aligned with the image grid. These channels are upsampled as needed and concatenated with the grayscale image to form the detector input, which is processed by a Faster~R--CNN-style architecture with a feature pyramid and RoI heads under an image-level MIL objective. In Stage~II, every proposal and classification decision is conditioned jointly on standard convolutional features and on the stable morphological structure captured by the wavelet--persistence maps.

\begin{figure*}[!t]
\centering
\includegraphics[width=\textwidth]{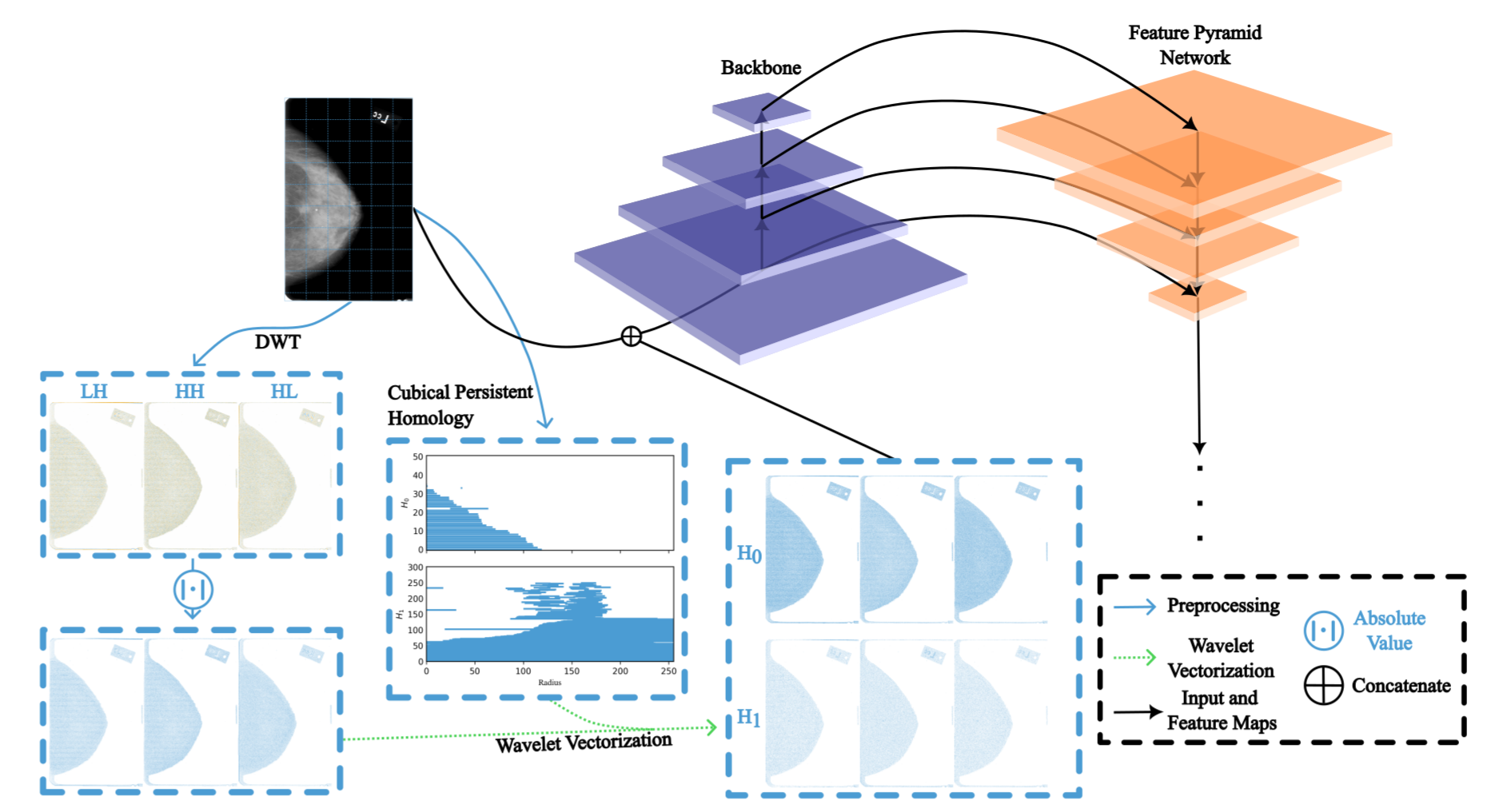}
\caption{During stage 2, the gated wavelet-persistence maps are upsampled and fused with backbone feature maps using feature-pyramid and RoI modules in a two-stage detection pipeline. Topological channels are injected by direct input-level concatenation, so detections and topological responses are evaluated on the same image domain.}
\label{fig:method_framework_stage2}
\end{figure*}

\subsection{Motivation Behind the Proposed Vectorization}

Deep learning models for mammography now achieve strong performance on curated, within--dataset benchmarks, yet their behavior under distribution shift remains difficult to characterize in a principled and reproducible way. In particular, widely used post--hoc attribution methods such as saliency maps exhibit weak sensitivity to model parameters, limited reproducibility, and localization behavior that often deviates substantially from task--specific segmenters, which constrains their reliability under domain shift \cite{Arun2021SaliencyTrustworthiness,Fuhrman2022ExplainableAIReview}. At the same time, clinical deployment requires that both predictions and their supporting cues remain interpretable when scanners, acquisition protocols, or patient populations change \cite{doi:10.1161/CIRCIMAGING.122.014519,10208407,GAO2023102889}. A broad range of counterfactual, prototype--based, and model--agnostic explanation methods has therefore been proposed, yet their empirical behavior remains strongly task-- and dataset--dependent \cite{Wang2024ScoreBasedCounterfactuals,Cai2024CounterfactualVQA,Gall_e_2023,HASSAN2025109569,Kothari2013HistologyShapeFeatures}.

Topological Data Analysis (TDA) provides descriptors of image morphology that are provably stable under bounded perturbations of the underlying intensity function. In the cubical setting, $\ell_\infty$--bounded perturbations of the image induce bounded bottleneck and Wasserstein perturbations of the corresponding persistence diagrams \cite{CohenSteiner2007StabilityPD}. Persistence diagrams, however, reside in an abstract birth--death space and do not directly encode spatial correspondence in the image plane. Common diagram vectorizations such as persistence landscapes and persistence images embed diagrams into finite--dimensional feature spaces with Lipschitz stability \cite{bubenik2015statisticaltopologicaldataanalysis,10.5555/3122009.3122017}, but these representations function primarily as global or regional descriptors rather than as spatially aligned image--plane signals.

Here, we construct a persistence--diagram vectorization that is explicitly aligned with a classical wavelet pyramid, producing spatial, multi--scale channels whose indices correspond to wavelet subbands rather than to abstract resolution levels. This coupling assigns persistent topological lifetimes to specific spatial locations and scales in the image. Wavelets are well matched to mammographic morphology, where diagnostically relevant structures such as microcalcifications, spiculated margins, and architectural distortion exhibit characteristic directional and multi--scale behavior. The resulting wavelet--persistence maps therefore encode stable morphological information in a form that is directly compatible with convolutional feature hierarchies.

These maps are used as auxiliary conditioning channels in a fixed two–stage mammography detector built on a single probe backbone (ConvNeXt–Tiny) and a deliberately limited training budget. This setup isolates the contribution of stable, structured morphological information by holding the architecture, loss, and optimization schedule constant. Under this protocol, the wavelet–persistence channels function as a morphology-aware input representation whose empirical value is evaluated solely under domain shift across acquisition regimes.

\subsection{Technical Details}

\subsubsection{Vectorization of the Persistence Diagram}
We now describe the wavelet-persistence vectorization that produces spatial, multi-scale channels from a persistence diagram.

Resize $I$ (aspect-preserving, with zero padding as needed) to $\widehat I\in[0,1]^{S\times S}$. A $J$-level orthogonal \texttt{Haar} transform yields the subbands
\[
\mathcal{P}_J=\Bigl\{\mathrm{LL}_J,\ (\mathrm{LH}_j,\mathrm{HL}_j,\mathrm{HH}_j)_{j=1}^J\Bigr\},
\]
with $\psi_{i,j}\in\mathbb{R}^{(S/2^j)\times(S/2^j)}$ and
$\mathrm{LL}_J\in\mathbb{R}^{(S/2^J)\times(S/2^J)}$.
To compare wavelet values with diagram coordinates $(b,d)\in[0,1]^2$, we fix a
monotone $1$-Lipschitz map $\tau:\mathbb{R}\to[0,1]$ and set
$\tilde\psi_{i,j}=\tau(\psi_{i,j})$. This aligns units and does not affect $D(\widehat I)$.

For $(b,d)\in(0,1)^2$ with $b<d$, $\varepsilon>0$, and a scalar
$\tilde\psi\in[0,1]$, define
\[
w(\tilde\psi;b,d)=
\begin{cases}
\dfrac{(\tilde\psi-b)(d-\tilde\psi)}{(d-b)+\varepsilon}, & \text{if } b\le \tilde\psi<d,\\[1ex]
0, & \text{otherwise.}
\end{cases}
\]
For each subband $\tilde\psi_{i,j}$ on $\Omega_{i,j}=\{1,\ldots,S/2^j\}^2$,
the per-pair map $W_{i,j}(\cdot;b,d)$ is obtained by pointwise application of $w$.
We aggregate by \emph{sum}:
\[
\mathsf{W}_{i,j}(D)\;=\;\sum_{(b,d)\in D} W_{i,j}(\cdot;b,d)\in\mathbb{R}^{\Omega_{i,j}}.
\]
The final wavelet-based representation stacks a subset of subbands; for example, for $J=3$ we may take
\[
\mathsf{W}(D)
=
\bigl(\mathsf{W}_{\alpha,j}(D)\bigr)_{\alpha\in\{\mathrm{LH},\mathrm{HL},\mathrm{HH}\},\, j\in\{2,3\}}
\;\cup\;
\bigl(\mathsf{W}_{\mathrm{LL},3}(D)\bigr).
\]
In our implementation, these wavelet-aligned maps are complemented by two additional baseline channels constructed directly from the original image (using the same gating function on $H_0$ and $H_1$ diagrams), composing eight topological channels in total (see Section \ref{sec:experiments-and-results}).

\subsubsection{Stability of the Vectorization}

\begin{lemma}\label{lem:partials}
Fix $\tilde\psi\in[0,1]$. On $\{(b,d):\, b<\tilde\psi<d\}$,
\[
\left|\frac{\partial w}{\partial b}\right|
=\frac{(d-\tilde\psi)^2}{(d-b+\varepsilon)^2}\le 1,
\qquad
\left|\frac{\partial w}{\partial d}\right|
=\frac{(\tilde\psi-b)^2}{(d-b+\varepsilon)^2}\le 1,
\]
and $w\equiv 0$ on the complement of $\{(b,d):\, b\le\tilde\psi<d\}$.
Moreover, $w$ extends continuously to $[0,1]^2$ with $w(\tilde\psi;b,b)=0$.
\end{lemma}

\begin{proof}
Inside $b<\tilde\psi<d$, write
$w=\bigl((\tilde\psi-b)(d-\tilde\psi)\bigr)/(d-b+\varepsilon)$.
Differentiate:
\[
\frac{\partial w}{\partial b}
=\frac{-(d-\tilde\psi)(d-b+\varepsilon)-(\tilde\psi-b)(-1)}{(d-b+\varepsilon)^2}
=-\,\frac{(d-\tilde\psi)^2}{(d-b+\varepsilon)^2},
\]
and
\[
\frac{\partial w}{\partial d}
=\frac{(\tilde\psi-b)(d-b+\varepsilon)-(\tilde\psi-b)(1)}{(d-b+\varepsilon)^2}
=\frac{(\tilde\psi-b)^2}{(d-b+\varepsilon)^2}.
\]
Since $0\le \tilde\psi-b\le d-b$ and $0\le d-\tilde\psi\le d-b$, each ratio is
$\le 1$. Outside the strip $b\le\tilde\psi<d$, $w\equiv 0$. At $b=\tilde\psi$
or $\tilde\psi\uparrow d$ the numerator vanishes, so continuity holds; for
$b=d$ we define $w=0$, matching the limit as $\varepsilon>0$ is fixed.
\end{proof}

\begin{corollary}\label{cor:lipschitz}
For any $(b,d),(b',d')\in[0,1]^2$ and fixed $\tilde\psi\in[0,1]$,
\[
\bigl|w(\tilde\psi;b,d)-w(\tilde\psi;b',d')\bigr|
\ \le\
\begin{cases}
\ \|(b,d)-(b',d')\|_{1},\\[0.25ex]
\ \sqrt{2}\,\|(b,d)-(b',d')\|_{2},\\[0.25ex]
\ 2\,\|(b,d)-(b',d')\|_{\infty}.
\end{cases}
\]
\end{corollary}

\begin{proof}
By Lemma~\ref{lem:partials}, on the region where $w$ is nonzero,
$\|\nabla_{(b,d)}w\|_{\infty}\le 1$. Duality of norms gives the Lipschitz
constants: for a differentiable path connecting $(b,d)$ to $(b',d')$,
\[
|w(\cdot)-w(\cdot)|
\le \int \|\nabla w\|_{\infty}\,\|(\dot b,\dot d)\|_{1}\,dt
\le \|(b,d)-(b',d')\|_{1}.
\]
Since $\|\nabla w\|_2\le \sqrt{2}\|\nabla w\|_\infty\le \sqrt{2}$ and
$\|\nabla w\|_1\le 2\|\nabla w\|_\infty\le 2$, the $\ell_2$ and $\ell_\infty$
bounds follow similarly.
\end{proof}

\begin{theorem}\label{thm:global}
Let $\|\cdot\|_{F}$ denote the Frobenius norm and
$|\Omega_{i,j}|=(S/2^j)^2$. For $p\in\{1,2,\infty\}$, let
$L_p\in\{1,\sqrt{2},2\}$ be the Lipschitz constant from Corollary~\ref{cor:lipschitz}.
For any two diagrams $D,D'$ (in a fixed homological dimension),
\[
\bigl\|\mathsf{W}_{i,j}(D)-\mathsf{W}_{i,j}(D')\bigr\|_{F}
\ \le\ L_p\,\sqrt{|\Omega_{i,j}|}\;\,W_{1}^{(p)}(D,D'),
\]
where $W_1^{(p)}$ is the $1$-Wasserstein distance on diagrams with ground cost
$\|(b,d)-(b',d')\|_{p}$.
\end{theorem}

\begin{proof}
Let $\gamma$ be an optimal matching for $W_1^{(p)}(D,D')$. By the triangle
inequality and linearity of the sum aggregation,
\[
\begin{split}
\bigl\|\mathsf{W}_{i,j}(D)-\mathsf{W}_{i,j}(D')\bigr\|_{F}
\;\le\;
\sum_{((b,d),(b',d'))\in\gamma} \\
\bigl\|W_{i,j}(\cdot;b,d)-W_{i,j}(\cdot;b',d')\bigr\|_{F}.
\end{split}
\]
Pointwise, Corollary~\ref{cor:lipschitz} gives
$|W_{i,j}(x;b,d)-W_{i,j}(x;b',d')|\le L_p\,\|(b,d)-(b',d')\|_{p}$ for every
pixel $x\in\Omega_{i,j}$. Hence
$\|W_{i,j}(\cdot;b,d)-W_{i,j}(\cdot;b',d')\|_{F}
\le L_p\,\sqrt{|\Omega_{i,j}|}\,\|(b,d)-(b',d')\|_{p}$.
Summing over $\gamma$ yields the claim.
\end{proof}

\begin{corollary}\label{cor:stack}
Let $\|\cdot\|_2$ be the Euclidean norm on the direct sum of selected subbands
and set $C_{J,p}=L_p\bigl(\sum_{(i,j)\in\mathcal{S}}|\Omega_{i,j}|\bigr)^{1/2}$,
where $\mathcal{S}$ is the index set of stacked subbands. Then
\[
\begin{split}
\bigl\|\mathsf{W}(D)-\mathsf{W}(D')\bigr\|_{2}
\;\le\;
C_{J,p}\;W_{1}^{(p)}(D,D') \\
\text{for }p\in\{1,2,\infty\}.
\end{split}
\]
\end{corollary}

The regularizer $\varepsilon>0$ prevents blow-up near the diagonal $b=d$ and tightens the derivative bounds uniformly. If $D=\varnothing$, then $\mathsf{W}_{i,j}(D)\equiv 0$ by the convention that the sum over the empty set is $0$. Combined with the classical bottleneck and Wasserstein stability of persistence diagrams under $\|\cdot\|_\infty$ perturbations of the image \cite{CohenSteiner2007StabilityPD}, Corollary~\ref{cor:stack} establishes that the stacked wavelet--persistence maps depend in a globally Lipschitz manner on small perturbations of the underlying persistence diagrams.

\subsection{Implementation Details}
Models are implemented in PyTorch with ImageNet-pretrained backbones. Multiscale transforms use PyWavelets, and persistent homology uses GUDHI. All experiments use a fixed, small training budget: Stage~I runs for 625 iterations with batch size $8$, and Stage~II runs for 100 iterations with batch size $1$, identical across all ablations. Optimization uses Adam with learning rate $10^{-4}$. Code for the wavelet-persistence vectorization and all experiments is available at \href{https://github.com/cfanning8/WaveletVectorization}{github.com/cfanning8/WaveletVectorization}.

\section{Experiments and Results}\label{sec:experiments-and-results}

\subsection{Datasets}\label{sec:datasets}

We evaluate on three public mammography cohorts spanning digitized film and full-field digital mammography (Table~\ref{tab:cohorts}). CBIS--DDSM \cite{SawyerLee2016CBISDDSM}, collected in the United States, contains 1,566 patients and 3,103 images and is used for training and validation under an 80/20 patient-level split stratified by malignancy, with no subject overlap. INbreast \cite{moreira2012inbreast}, collected in Portugal, contains 108 patients and 410 images and serves as the primary external test set. CMMD \cite{Cui2021CMMD}, collected in China, contains 1,534 patients and 4,798 images and serves as a secondary external test set for cross-population and cross-acquisition evaluation. All images are processed in native grayscale and normalized per image to $[0,1]$.

\begin{table}[t]
\centering
\footnotesize
\caption{Summary of mammography cohorts.}
\label{tab:cohorts}
\renewcommand{\arraystretch}{1.12}

\begin{tabular}{@{}l c@{}}
\hline
\textbf{INbreast} & \textbf{Value} \\
\hline
Country & Portugal \\
Modality & FFDM \\
Patients & 108 \\
Images & 410 \\
Usage & Testing I \\
Resolution & 3328$\times$2560 \\
\hline
\end{tabular}

\vspace{2mm}

\begin{tabular}{@{}l c@{}}
\hline
\textbf{CBIS--DDSM} & \textbf{Value} \\
\hline
Country & USA \\
Modality & Digitized Film \\
Patients & 1{,}566 \\
Images & 3{,}103 \\
Usage & Training and Validation \\
Resolution & 5086$\times$3036 \\
\hline
\end{tabular}

\vspace{2mm}

\begin{tabular}{@{}l c@{}}
\hline
\textbf{CMMD} & \textbf{Value} \\
\hline
Country & China \\
Modality & FFDM \\
Patients & 1{,}534 \\
Images & 4{,}798 \\
Usage & Testing II \\
Resolution & 2294$\times$1914 \\
\hline
\end{tabular}

\vspace{1mm}
\begin{minipage}{0.98\linewidth}
\end{minipage}
\end{table}

\subsection{Training Protocol and Architecture}\label{sec:training-protocol}

All models are implemented in PyTorch using an ImageNet--pretrained ConvNeXt--Tiny backbone. Training follows a fixed two--stage curriculum with no architecture changes, no schedule tuning, and no domain adaptation.

In Stage~I, patch--level multiple--instance learning operates on $224{\times}224$ grayscale patches with batch size $8$ for $625$ iterations on CBIS--DDSM. Patients with malignant and benign labels are sampled with equal probability to control class imbalance. This stage produces patch--level malignancy scores and a coarse image--level prior.

In Stage~II, RoI--level detection with image--level MIL fine--tunes a Faster~R--CNN detector with the same ConvNeXt--Tiny backbone under a fixed budget of $K{=}100$ iterations on CBIS--DDSM with batch size $1$. Optimization uses Adam with learning rate $10^{-4}$ and parameters $(\beta_1,\beta_2)=(0.9,0.999)$. These iteration counts, optimizer settings, and data splits remain fixed across all experiments.

Each mammogram $X$ is min--max normalized to $[0,1]$. When topology is enabled, cubical persistence diagrams in $H_0$ and $H_1$ are computed on wavelet--transformed versions of $X$, and the wavelet--persistence vectorization $\mathcal{W}$ from Section~\ref{sec:methods} produces a topological map
\[
T(X)\in\mathbb{R}^{8\times H\times W},
\]
with six wavelet--persistence channels and two baseline persistence channels. The detector input is formed by direct channel concatenation
\[
Z(X)=[X\Vert T(X)]\in\mathbb{R}^{9\times H\times W},
\]
and replaces the single--channel grayscale input at the backbone input. When topology is disabled, the detector receives $X$ alone. No other changes are made to the detection architecture.

For auxiliary analysis only, we define a pooled topological embedding
\[
z_{\text{topo}}(X)\in\mathbb{R}^8
\]
by spatial averaging of the eight channels of $T(X)$ with patient--level aggregation. This vector is used exclusively for linear probes and embedding comparisons and does not enter the detector. By Theorem~\ref{thm:global} and Corollary~\ref{cor:stack}, the mapping $T$ is globally Lipschitz with respect to $\ell_\infty$ perturbations of $X$.

\subsection{Evaluation}\label{sec:evaluation-protocol}

All splits are defined at the patient level. CBIS--DDSM is divided into an 80/20 train-validation split stratified by malignancy. INbreast and CMMD are held out entirely as external test sets and are never used for tuning.

For each model, view--level scores are aggregated to the patient level by taking the maximum score across both breasts and both standard views (CC and MLO), so that any positive view or breast yields a positive patient score. We report ROC--AUC and, at a single fixed operating point, sensitivity, specificity, and accuracy. The operating threshold is selected once on the CBIS--DDSM validation split by maximizing Youden’s index and is then frozen for evaluation on INbreast and CMMD.

Uncertainty is estimated using 10{,}000 patient--wise bootstrap resamples, with 95\% confidence intervals computed from percentile bounds. For selected model pairs, we also bootstrap the AUC difference to assess whether the topology--conditioned variant outperforms its grayscale counterpart.

\subsection{Malignancy Classification}

Each mammogram is min--max normalized to $[0,1]$ prior to processing. When topology is enabled, the grayscale image is augmented with eight wavelet--persistence channels, producing a nine--channel input formed by simple channel concatenation. The topological channels are rescaled to $[0,1]$ and used directly as auxiliary image features.

In all experiments, topological information is injected only at the input: the detector backbone receives either the single--channel grayscale image $X$ or the nine--channel tensor $Z(X) = [X \Vert T(X)]$ formed by concatenating $X$ with the wavelet--persistence maps $T(X)$. No FiLM modulation of intermediate feature maps or RoI--level topological descriptors are used in the reported pipeline.

Wavelet--persistence maps are constructed using orthogonal \texttt{Haar} and \texttt{Daubechies} wavelets of orders two and four, with decomposition depths one through three and periodization boundary conditions. Cubical persistent homology is computed in dimensions $H_0$ and $H_1$. All $H_0$ pairs are retained except for the single global background component. For $H_1$, only a fixed fraction of the lowest--persistence pairs is preserved, using retention levels of 10\%, 25\%, and 50\%. A smooth gating threshold of $\varepsilon=10^{-6}$ is used throughout.

\subsubsection{Wavelet--Persistence Ablations on CMMD}

We assess the sensitivity of the wavelet--persistence construction on CMMD using a ConvNeXt--Tiny backbone in Stage~II under a fixed training protocol. All models are trained exclusively on CBIS--DDSM, evaluated on CMMD at the patient level, and use the same operating threshold selected on the CBIS--DDSM validation split. Table~\ref{tab:ablation_cmmd_auc} reports patient--level AUC for controlled variations in wavelet family, decomposition depth, and $H_1$ keep-percentage.

\begin{table}[t]
\centering
\caption{Wavelet--persistence ablations on CMMD at the patient level.}
\label{tab:ablation_cmmd_auc}
\renewcommand{\arraystretch}{1.15}
\begin{tabular}{@{}l c@{}}
\hline
\textbf{Ablation} & \textbf{AUC} \\
\hline
\multicolumn{2}{@{}l}{\emph{Wavelet family}} \\
\hline
\texttt{haar}  & $0.6073 \pm 0.0319$ \\
\texttt{db2}   & $0.5817 \pm 0.0328$ \\
\texttt{db4}   & $0.6045 \pm 0.0331$ \\
\hline
\multicolumn{2}{@{}l}{\emph{Wavelet depth}} \\
\hline
$J=1$ & $0.5976 \pm 0.0319$ \\
$J=2$ & $0.6291 \pm 0.0329$ \\
$J=3$ & $0.5749 \pm 0.0320$ \\
\hline
\multicolumn{2}{@{}l}{\emph{$H_1$ keep-percentage}} \\
\hline
$h_1 = 0.10$ & $0.5832 \pm 0.0320$ \\
$h_1 = 0.25$ & $0.6021 \pm 0.0328$ \\
$h_1 = 0.50$ & $0.6274 \pm 0.0319$ \\
\hline
\end{tabular}
\end{table}

As shown in Table~\ref{tab:ablation_cmmd_auc}, AUC varies only weakly across wavelet families, with \texttt{haar} and \texttt{db4} performing similarly and \texttt{db2} consistently lower. Decomposition depth and $H_1$ retention exhibit a stronger effect: $J{=}2$ achieves the highest AUC among the tested depths, and increasing the $H_1$ keep-percentage produces a monotonic increase in performance, with the maximum attained at $h_1{=}0.50$. Under this fixed transfer protocol, external AUC on CMMD is therefore far more sensitive to scale selection and $H_1$ mass retention than to the choice of orthogonal wavelet family.

\subsection{Topology--Only Linear Probes}\label{sec:experiments-linear-probes}

We next ask whether the stable embedding $z_{\text{topo}}(X)$ carries standalone information about malignancy, independent of any detector. For each patient we compute $z_{\text{topo}}(X)\in\mathbb{R}^8$ on CBIS--DDSM and fit a logistic regression classifier
\[
\hat{y} = \sigma\big(w^\top z_{\text{topo}}(X) + b\big)
\]
trained on CBIS--DDSM and evaluated on INbreast.

On INbreast, the topology--only linear probe attains a patient--level AUC of $0.61 \pm 0.11$. This shows that the pooled wavelet--persistence representation already carries a weakly discriminative geometry for malignancy in this cohort, without any learned visual features.

\subsection{Architecture-Topology Compatibility}\label{sec:experiments-architecture-compatibility}

Table~\ref{tab:convnext_main} reports patient--level AUC for ConvNeXt--Tiny on INbreast with and without wavelet--persistence conditioning under an identical training budget. Without topology, ConvNeXt--Tiny operates near chance ($0.55 \pm 0.11$). With wavelet--persistence channels concatenated at the input, AUC increases sharply to $0.75 \pm 0.09$. This shift moves the model from negligible discrimination to clearly nontrivial patient--level performance without any change to the backbone, loss, optimizer, or training schedule.

\begin{table}[t]
\centering
\caption{ConvNeXt--Tiny patient--level AUC on INbreast.}
\label{tab:convnext_main}
\renewcommand{\arraystretch}{1.05}
\begin{tabular}{@{}l c@{}}
\hline
\textbf{Input} & \textbf{INbreast} \\
\hline
Grayscale      & $0.55 \pm 0.11$ \\
Topology       & $0.75 \pm 0.09$ \\
\hline
\end{tabular}
\end{table}

This experiment isolates the effect of stable topological side information under fixed optimization and fixed source--domain training. Because the architecture, data splits, and iteration budget are held constant, the observed gain reflects a genuine interaction between ConvNeXt--Tiny’s feature hierarchy and the spatial, multi--scale morphology encoded by the wavelet--persistence maps. All subsequent analyses, therefore, treat ConvNeXt--Tiny as a controlled probe architecture for studying how persistent--homology--based conditioning reshapes early--regime external generalization.

\subsection{Topological Embedding Distances Across Cohorts}\label{sec:experiments-domain-shift}

For each cohort, we compute the pooled topological embedding $z_{\text{topo}}(X) \in \mathbb{R}^8$ at the patient level and measure the Wasserstein--2 distance between its empirical distribution and that of CBIS--DDSM. In this embedding, the distance from CBIS--DDSM to INbreast is $0.46$, and the distance from CBIS--DDSM to CMMD is $0.31$.

\section{Discussion}\label{sec:discussion}

The experiments in Section~\ref{sec:experiments-and-results} assess wavelet--persistence conditioning under a fixed training schedule, a fixed source domain (CBIS--DDSM), and a single controlled probe architecture (ConvNeXt--Tiny), with external evaluation on INbreast and CMMD. Within this setting, the effect of the auxiliary channels is cohort dependent.

On INbreast, patient-level AUC increases from $0.55\pm0.11$ to $0.75\pm0.09$ under identical Stage~I and Stage~II optimization budgets. This shift moves ConvNeXt--Tiny from near-chance discrimination to clearly nontrivial performance without modifying the backbone, loss, optimizer, or schedule. The improvement arises strictly from adding the wavelet--persistence maps at the input and shows that the backbone can use stable, multiscale morphological information effectively in early training.

CMMD highlights a different aspect of the construction. The controlled ablations show that external AUC on this cohort is most sensitive to the choice of wavelet depth and to the mass retained in $H_1$, with comparatively smaller variation across wavelet families. In the pooled topological embedding, INbreast and CMMD occupy different geometric positions relative to CBIS--DDSM, reflecting differences in acquisition process and image texture. Together, these results indicate that the utility of the topological channels is shaped by how their induced morphology interacts with cohort-specific image characteristics.

Although the vectorization $\mathsf{W}$ is globally Lipschitz with respect to Wasserstein perturbations of persistence diagrams (Theorem~\ref{thm:global}, Corollary~\ref{cor:stack}), the diagrams themselves depend on the image-formation process. Because acquisition regimes differ in their underlying image statistics, stability governs sensitivity to perturbations within a given regime but does not imply invariance across regimes.

Several limitations follow from the protocol. Operating thresholds are selected on a small CBIS--DDSM validation split and carried over to the external cohorts, which affects sensitivity--specificity tradeoffs. The analysis focuses on patient-level AUC and on a single ImageNet-pretrained backbone under tightly constrained computation. Broader evaluations—including alternative topological encodings, multi-parameter filtrations, calibration strategies, and explicit domain adaptation—remain open directions.

\section{Conclusion}

We introduced a wavelet--based vectorization of cubical persistent homology that produces spatial, multi--scale topological maps aligned with the image grid and proved that this construction is globally Lipschitz with respect to the $1$--Wasserstein distance on persistence diagrams. The vectorization assigns birth--death lifetimes to wavelet subbands through an explicit gating function, admits uniform derivative bounds, and yields stacked spatial feature maps with provable stability under diagram perturbations. This establishes a mathematically controlled mechanism for converting persistent homology into image--plane signals compatible with convolutional architectures.

Experimentally, we evaluated this construction under a fixed two--stage pipeline with a single controlled probe architecture (ConvNeXt--Tiny). Stage~I performs patch--level multiple--instance learning on CBIS--DDSM, and Stage~II fine--tunes a Faster~R--CNN detector under a fixed iteration budget. Wavelet--persistence maps are concatenated directly at the detector input without modifying the backbone, loss, or optimizer. Training uses digitized film CBIS--DDSM only, with external testing on full--field digital INbreast and CMMD at the patient level under a fixed operating threshold.

The primary limitation of this study is the shallow scope of the architectural ablation. The analysis fixes a single backbone and a single optimization regime. Broader coverage across larger ConvNeXt variants and across fundamentally different backbone families would further clarify how architectural scale and inductive bias interact with stable topological side channels.

A natural direction for future work is to integrate topological information into the detector architecture rather than supply it solely as an auxiliary input. Possible architectural points of insertion are feature-map modulation and the region proposal network to route topological information directly into the detector’s learned representations.

\section*{References}
\bibliographystyle{IEEEtran}
\bibliography{refs}

\end{document}